\documentclass[10pt,journal,compsoc]{IEEEtran}
\usepackage{algorithmic}
\usepackage{algorithm}
\usepackage{array}
\usepackage[caption=false,font=normalsize,labelfont=sf,textfont=sf]{subfig}
\usepackage{textcomp}
\usepackage{stfloats}
\usepackage{url}
\usepackage{cite}
\usepackage{multirow}
\newcommand{\myparagraph}[1]{\vspace{0.1em}\noindent\textbf{#1}}
\usepackage[colorlinks,hyperindex,breaklinks]{hyperref}
\newcommand{\ie}{\textit{i}.\textit{e}.}
\newcommand{\eg}{\textit{e}.\textit{g}.}

\usepackage{pifont}

\usepackage{color}
\usepackage{graphicx}
\usepackage{ragged2e}
\usepackage{arydshln}
\usepackage{amsmath}
\usepackage{amsthm}
\newtheorem{theorem}{Theorem}[section]
\newtheorem{lemma}[theorem]{Lemma}
\usepackage{float}
\usepackage{makecell}
\usepackage{booktabs}
\usepackage[english]{babel}
\begin{document}
\title{Generalized Task-Driven Medical Image Quality Enhancement with Gradient Promotion}
\author{Dong~Zhang,~\IEEEmembership{Member,~IEEE},
        Kwang-Ting~Cheng,~\IEEEmembership{Fellow,~IEEE}
\IEEEcompsocitemizethanks{
\IEEEcompsocthanksitem D. Zhang and K-T. Cheng are with the Department of Electronic and Computer Engineering, The Hong Kong University of Science and Technology, Hong Kong, China. E-mail:~\{dongz, timcheng\}@ust.hk.
}
}
\markboth{}%
{Shell \MakeLowercase{\textit{et al.}}: Bare Demo of IEEEtran.cls for Computer Society Journals}
\IEEEtitleabstractindextext{
\begin{abstract}
\justifying
Thanks to the recent achievements in task-driven image quality enhancement (IQE) models like ESTR~\cite{liu2022exploring}, the image enhancement model and the visual recognition model can mutually enhance each other's quantitation while producing high-quality processed images that are perceivable by our human vision systems. However, existing task-driven IQE models tend to overlook an underlying fact -- different levels of vision tasks have varying and sometimes conflicting requirements of image features. To address this problem, this paper proposes a generalized gradient promotion (\emph{GradProm}) training strategy for task-driven IQE of medical images. Specifically, we partition a task-driven IQE system into two sub-models, \ie, a mainstream model for image enhancement and an auxiliary model for visual recognition. During training, \emph{GradProm} updates only parameters of the image enhancement model using gradients of the visual recognition model and the image enhancement model, but only when gradients of these two sub-models are aligned in the same direction, which is measured by their cosine similarity. In case gradients of these two sub-models are not in the same direction, \emph{GradProm} only uses the gradient of the image enhancement model to update its parameters. Theoretically, we have proved that the optimization direction of the image enhancement model will not be biased by the auxiliary visual recognition model under the implementation of \emph{GradProm}. Empirically, extensive experimental results on four public yet challenging medical image datasets demonstrated the superior performance of \emph{GradProm} over existing state-of-the-art methods.
\end{abstract}
}

\maketitle
\IEEEdisplaynontitleabstractindextext
\IEEEpeerreviewmaketitle
\section{Introduction}
\IEEEPARstart{M}{edical} imaging performs an increasingly important role in modern medicine systems, which enables physicians to visualize internal anatomical structures and evaluate disease progression~\cite{varoquaux2022machine,zhang2022deep}. Medical image analysis (MIA) -- an essential component of medical imaging -- aims at providing an underlying pattern to analyze and interpret medical images for clinical decision-making~\cite{shen2017deep,anwar2018medical,zhang2022deep}. During past years, thanks to the tremendous progress of deep learning technology in computer vision tasks~\cite{he2016deep,han2022survey,long2015fully}, the community of MIA has also proliferated that enables clinicians to extract a wealth of semantic information from the given medical images and apply it to a wide range of practical applications, \eg, cancer diagnosis~~\cite{lin2019automated}, multi-organ segmentation~\cite{ma2022fast}, and robotic surgery~\cite{gao2021future}. 
\begin{figure*}[t]
\centering
\includegraphics[width=.99\textwidth]{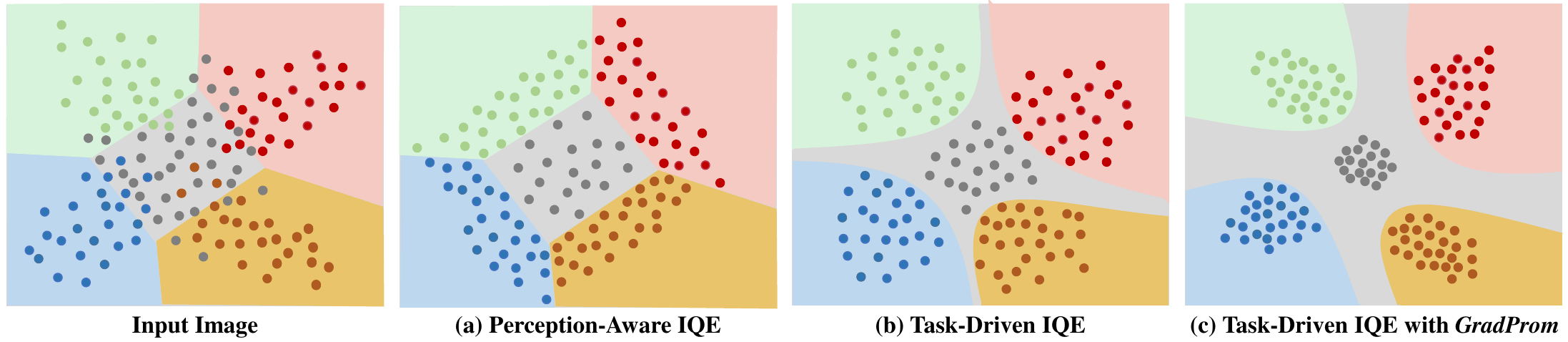}
\vspace{-4mm}
\caption{The evolution of medical image quality enhancement (IQE) in visual recognition. The IQE model is represented by dots of different colors in the foreground, indicating different categories, while the colorful block in the background represents the decision boundary in the recognition model. In perception-aware IQE (a), samples have a closer representation space, leading to no direct benefit for the decision boundary in the downstream IQE model. Task-driven IQE (b) employs data transfer or gradient transfer to enhance both the upstream IQE model and the downstream recognition model, resulting in a more compact representation space and a clear decision boundary. However, the feature requirements of these two models may be inconsistent (as shown in Figure~\ref{figr1} below), leading to sub-optimization. Our \emph{GradProm} in (c) overcomes this limitation, resulting in a more optimal outcome where the representation space is more compact and the decision boundary is clearer.}
\vspace{-3mm}
\label{fig1}
\end{figure*} 

Image quality matters for MIA~\cite{jawdekar2021review,jose2021image}, \ie, images of a higher quality are expected to produce more accurate recognition performance~\cite{ma2013dictionary,hou2023self,li2023transforming}. As a result, various medical image quality enhancement (IQE) approaches have been developed, \eg, super-resolution~\cite{li2021review,georgescu2023multimodal}, and image denoising~\cite{patil2022medical}. The input of these approaches is a low-quality (\ie, low-resolution or/and \emph{with} noise) medical image, and the output is a rendered medical image that meets some expected high-quality (\ie, high-resolution or/and \emph{without} noise) characteristics~\cite{jawdekar2021review,han2022survey,liu2022exploring}. Medical images processed by IQE approaches are then used as the input of downstream visual recognition models, \eg, image diagnose and semantic segmentation~\cite{varoquaux2022machine,mahapatra2019image}. Generally, an implicit consensus behind these medical image processing approaches is that the perception obtained by our visual systems, and evaluation metrics (\eg, PSNR and SSIM) of such a perception are important criteria for evaluating the image quality~\cite{setiadi2021psnr}. Therefore, these approaches mainly pursue an apparent high-quality presentation close to our human visual perception. For example,  ``human evaluation'' is an important practical criteria for the medical IQE task~\cite{shamshad2022transformers}.

However, the enhanced visual perceptual image quality does not equate to more beneficial information obtained by downstream visual recognition models~\cite{xie2015task,liu2022exploring,sharma2018classification,johnson2016perceptual}. The reasons can be roughly attributed to the following two aspects.
(\uppercase\expandafter{\romannumeral1})~Intuitively, the perception-aware medical image processing model is separated from the downstream visual recognition model in the training process (\ie, an IQE system with both image processing models and visual recognition models is not trained in an end-to-end manner)~\cite{han2022survey,jawdekar2021review}, therefore, the decision boundary of the downstream visual recognition model can not be directly affected by IQE approaches~\cite{liu2022exploring,mahapatra2019image,patil2022medical}. 
(\uppercase\expandafter{\romannumeral2})~Perception-aware methods have the potential to make image representations closer in manifold space. As shown in Figure~\ref{fig1} (a), due to some object-specific feature grids are removed from input images in the perception-aware medical IQE models (\eg, image denoising), different categories of samples fed into the downstream visual recognition model are actually closer (\emph{compared to the input image representation space}) in the representation space~\cite{jose2021image,patil2022medical,ma2013dictionary}, result in a closer decision boundary, which is somewhat harmful~\cite{hou2023self,talebi2018nima}. This standpoint can also be used to explain why even though the performance of super-resolution is improved, the performance of the downstream image classification is unexpectedly decreased~\cite{johnson2016perceptual,jaffe2017super,liao2018defense}.

Based on the above analysis, the proposals of the task-driven IQE paradigm becomes a matter of course~\cite{liu2022exploring}. Existing task-driven IQE paradigm jointly trains image enhancement models and visual recognition models\cite{liu2022exploring}. The goal of this paradigm is to improve the perceived image quality by our human visual systems while also enabling mutual quantitative performance improvements between IQE models and recognition models. This advanced paradigm has been shown to enhance recognition accuracy in both upstream and downstream tasks~\cite{xie2015task,johnson2016perceptual}. For example, ESTR~\cite{liu2022exploring} concatenates an image quality enhancement model (\ie, \emph{image super-resolution, denoising, and JPEG-deblocking}) and an visual recognition model (\ie, \emph{image classification and object detection}), and fixes parameters of the recognition model during training to avoid the performance crash. Experimental results validate a strong generalization of this method. 
However, task-driven IQE approaches overlook an important fact:~\emph{different levels of computer vision tasks have varying and sometimes conflicting demands for image features}~\cite{zhang2018visual,mahapatra2022interpretability,zhang2020causal,baslamisli2018joint,zeng2019joint}. This is also our key motivation. For example, as shown in Figure~\ref{figr1}, under the same given image in (a), the denoising task in (b) focuses on all regions of the given image, where each pixel contributes to its output. In comparison, the semantic segmentation task in (c) and the diagnosis task in (d) focus on the foreground region of the object and the discriminative local region of the foreground object, respectively.
Therefore, although the advanced task-driven IQE paradigm can alleviate the intrinsic problems of broken upstream and downstream models in the perception-aware IQE paradigm, as shown in Figure~\ref{fig1} (b), there are potential inconsistencies between the upstream image enhancement model and the downstream visual recognition model~\cite{li2019gradient,wang2020makes}, resulting in declined performance~\cite{liu2019auxiliary,zhu2022prompt}.
\begin{figure}[t]
\centering
\includegraphics[width=.48\textwidth]{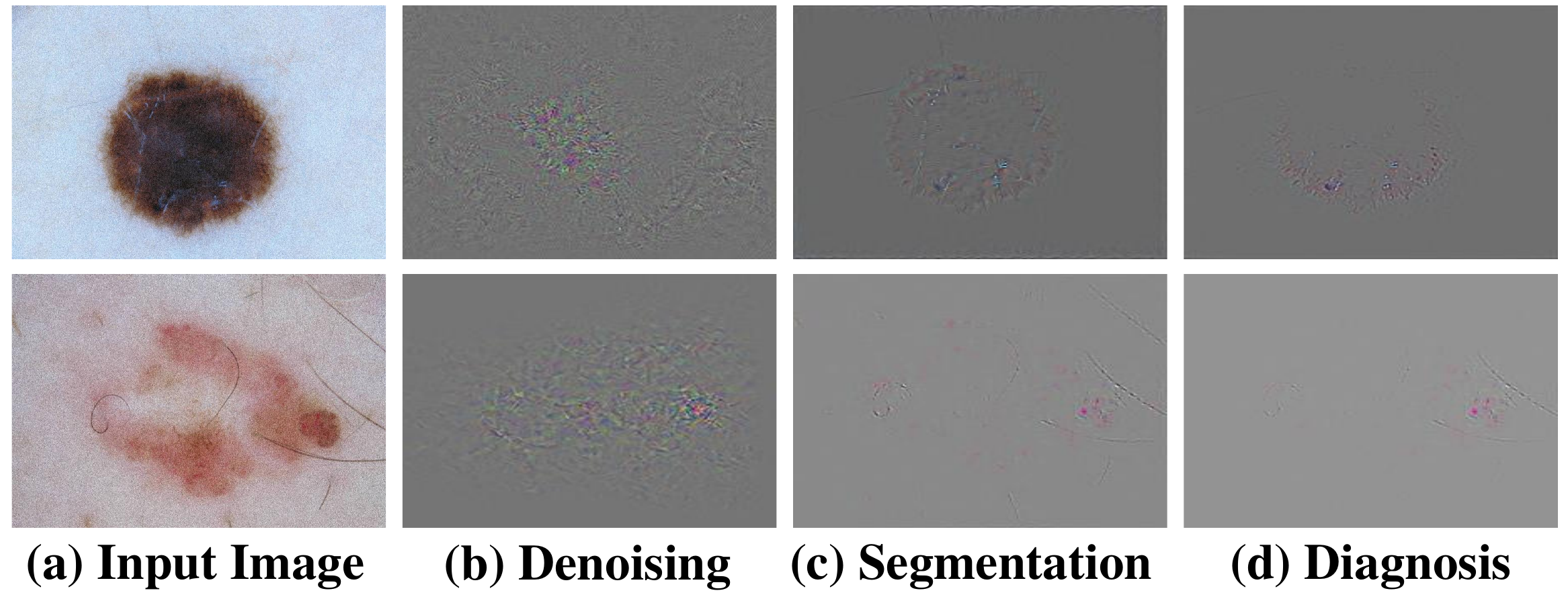}
\vspace{-2mm}
\caption{The guided backpropagation visualizations for different vision tasks. Under the same given image in (a), for denoising in (b), we use SR-ResNet~\cite{ledig2017photo}, for semantic segmentation in (c), we use UNet~\cite{ronneberger2015u}, and for diagnosis in (d), we use ResNet~\cite{he2016deep}. Sample images are from the ISIC 2018 dataset~\cite{codella2019skin}.}
\vspace{-4mm}
\label{figr1}
\end{figure} 
  
The story of task-driven IQE should be continued. To address limitations in the existing medical task-driven IQE paradigm, we propose a simple yet effective generalized gradient promotion (\emph{GradProm}) training strategy from the perspective of gradient-dominated parameter updates. 
"Generalized" refers to that the proposed \emph{GradProm} is applicable to a wide range of medical image modalities, rather than being limited to a specific image modality. 
Specifically, we first explicitly partition a task-driven IQE system into two sub-models -- \textit{a mainstream model for image enhancement} and \textit{an auxiliary model for visual recognition}. By ``promotion'', we mean that using the auxiliary recognition model to boost the optimization process of the mainstream image enhancement model. 
Because we care more about the image enhancement performance in the task-driven IQE paradigm. During the training process, \emph{GradProm} updates parameters of the mainstream image enhancement model using gradients of the visual recognition model and the image enhancement model, but only when gradients of these two sub-models are aligned in the same optimization direction ({\ie, \emph{the visual recognition model is useful}}), which is measured by their cosine similarity. In case gradients of these two sub-models are not in the same optimization direction ({\ie, \emph{the visual recognition model is not useful}}), \emph{GradProm} only uses the gradient of the image enhancement model to update its parameters. 
Theoretically, \emph{GradProm} ensures that the gradient descent direction of the image enhancement model will not be biased by the gradient descent direction of the visual recognition model, as illustrated in Figure~\ref{fig1} (c), resulting in a compact representation space (\emph{for the mainstream image enhancement model}) and an alienative decision boundary (\emph{for the auxiliary visual recognition model}). 
Our experiments involved publicly available synthetic and real medical datasets from diverse image modalities, such as ISIC 2018~\cite{codella2019skin}, COVID-CT~\cite{yang2020covid}, Lizard~\cite{graham2021lizard}, and CAMUS~\cite{leclerc2019deep}. The obtained experimental results demonstrate the superior performance of our proposed \emph{GradProm} approach over existing state-of-the-art methods for two challenging medical image enhancement tasks, namely denoising and super-resolution, as well as two fundamental medical visual recognition tasks, namely diagnosis and segmentation.

Our contributions of this work can be summarized as:
\begin{itemize}
\item We propose a task-driven paradigm for medical IQE and explicitly split the system into two sub-models, which is a novel mechanism to the problem.
\item We introduce \emph{GradProm}, a simple yet effective generalized training strategy that can dynamically train the two sub-models and achieve continuous performance improvements without requiring extra data or changes to network architecture.
\item We provide theoretical proof that \emph{GradProm} can converge to a local optimum without being biased towards the auxiliary visual recognition model.
\item We conduct experiments on four publicly available medical image datasets from diverse modalities, demonstrating that our proposed \emph{GradProm} achieves state-of-the-art performance on IQE.
\end{itemize}
\section{Related Work}
\subsection{Medical Image Quality Enhancement (IQE)}
Medical IQE is a critical area of image processing that aims to improve the visual appearance and diagnostic accuracy of medical images, such as X-rays, computerized tomography scans, magnetic resonance imaging scans, and 2D ultrasound images~\cite{shen2017deep,zhang2022deep}. In recent years, with the advent of deep learning technology, medical IQE has received significant attention and has been applied to various downstream applications, such as medical image diagnosis and multi-organ segmentation~\cite{mahapatra2019image,yan2020higcin,jose2021image,georgescu2023multimodal,varoquaux2022machine}. Existing medical IQE tasks can be broadly classified into two categories: image restoration and image enhancement. \emph{The first category}, image restoration, aims to improve the quality of degraded or noisy medical images~\cite{yi2019generative}. However, this category poses several challenges, including the presence of noise and artifacts in medical images and the development of methods that can handle different types of degradation and imaging modalities~\cite{lin2019automated,yan2023progressive,mahapatra2019image}. \emph{The second category}, image enhancement, aims to improve image contrast and/or sharpen image details. The commonly used approaches include contrast enhancement, edge enhancement, multi-scale analysis, and color correction~\cite{zhang2022new,shamshad2022transformers,mei2022deep}. In this paper, we focus on both medical image restoration and medical image enhancement tasks. Our contribution lies in proposing a generalized gradient-based model training strategy for the task-driven medical IQE paradigm. One of our notable advantages is that we improve the model's performance without changing its architecture or data input manner. Therefore, our method will not increase the computational complexity during the inference stage.
\subsection{Multi-Task Learning (MTL) and Auxiliary Learning}
MTL is a popular approach in computer vision that aims to leverage the useful knowledge from related tasks to improve the overall performance of all involved tasks~\cite{caruana1997multitask,zhang2018overview}. In recent years, MTL has been successfully applied to various computer vision applications and has been shown to be effective in enhancing the performance of individual tasks while reducing the amount of data and computation required~\cite{sener2018multi}. MTL treats all involved tasks with equal importance in a unified framework, attempting to train a computer vision system with shared feature representations to achieve balanced optimization for multiple tasks. 
Examples of tasks that can be tackled using MTL include object detection and segmentation~\cite{li2016deepsaliency,zhang2024cae}, image classification and landmark localization~\cite{chen2021residual}, and image captioning and text-to-image generation~\cite{stefanini2022show}. In particular, when multiple tasks are not equally important, the involved tasks can be classified into the mainstream task(s) and the auxiliary task(s). This setting is also referred to as auxiliary learning~\cite{liebel2018auxiliary,yu2020gradient}. In this case, only the results of the mainstream task are of concern, and the auxiliary task is used only to assist the main tasks in achieving better performance~\cite{lyle2021effect,shi2020auxiliary}.
Auxiliary learning has been commonly used in generative adversarial networks, such as classification and jigsaw solvers~\cite{liu2019self,lin2019adaptive}. In this paper, we propose a task-driven medical image quality enhancement system as an auxiliary learning paradigm, where image processing is the mainstream task, and image recognition is the auxiliary task. Our contribution is to introduce a gradient promotion strategy that uses the auxiliary recognition model to boost the performance of the mainstream image processing model. Furthermore, our method is not limited to a specific image modality and is applicable to various medical image modalities, demonstrating strong generalization capabilities.
\section{Methodology}
\subsection{Preliminaries}
\label{sec3.1}
Following~\cite{li2021review,georgescu2023multimodal,liu2022exploring,patil2022medical,zhang2022new}, the task-driven medical IQE paradigm is intrinsically an image enhancement task, where the input is a low-quality image $X$ and the corresponding high-quality image $Y$ as its label. The training process is to make $X$ as close as possible to $Y$ after being encoded by the the image enhancement model $IP$ and the visual recognition model $VR$, where the input of $VR$ is the output of $IP$~\cite{liu2022exploring}. The whole system is jointly trained in an end-to-end fashion with pixel-wise loss functions of $IP$ and $VR$, respectively. Therefore, the total loss $L_{total}$ for a task-driven medical IQE system can be expressed as:
\begin{equation}
L_{total} = L_{IP} + \lambda~L_{VR}, 
\label{eq1} 
\end{equation}
where $L_{IP}$ denotes the image enhancement loss, and $L_{VR}$ denotes the visual recognition loss. 
$\lambda$ is a hyper-parameter that is used to control the loss balance between $L_{IP}$ and $L_{VR}$. Specifically, $L_{IP}$ can be formulated as:
\begin{equation}
L_{IP} = F_{loss} (IP(X), Y), 
\label{eq2} 
\end{equation}
where $F_{loss}(\cdot)$ denotes a pixel-wise loss function (\eg, mean absolute error loss, mean squared error loss, and L1/L2 loss~\cite{patil2022medical}). Besides, an intuitive $L_{VR}$ can be formulated as:
\begin{equation}
L_{VR} = F_{loss} (VR(IP(X)), VR(Y)).
\label{eq3} 
\end{equation}
The assumption behind Eq.~\ref{eq3} is that we do not have semantic labels (\ie, image-level class labels for image diagnose or the pixel-level masks for semantic segmentation) for $VR$, \ie, it is based on the unsupervised setting. At the same time, a task-driven medical IQE system can be formulated as follows if we have the semantic label $S$ of $VR$:
\begin{equation}
L_{VR} = F_{loss} (VR(IP(X)), S),
\label{eq4} 
\end{equation}
\ie, this is a supervised formulation of $VR$. In what follows, to distinguish between the unsupervised and the supervised $L_{VR}$, we formulate the unsupervised loss function and the supervised loss function as $L^u_{VR}$ and $L^s_{VR}$, respectively.
Although the task-driven medical IQE paradigm can associate the mainstream and the auxiliary models in an uniform system, it overlooked an important fact: different levels of tasks have varying and sometimes conflicting demands for image features~\cite{zhang2018visual,mahapatra2022interpretability,zhang2020causal}, \ie, there are potential inconsistencies between $IP$ and $VR$~\cite{li2019gradient,wang2020makes}. Therefore, the joint training strategy that ignores the differences in requirements of different computer vision models for features will bring the declined performance~\cite{liu2019auxiliary,zhu2022prompt}. In the following, we will propose a simple yet generalized gradient-based training strategy to solve this problem.
\begin{figure}[t]
\centering
\includegraphics[width=.48\textwidth]{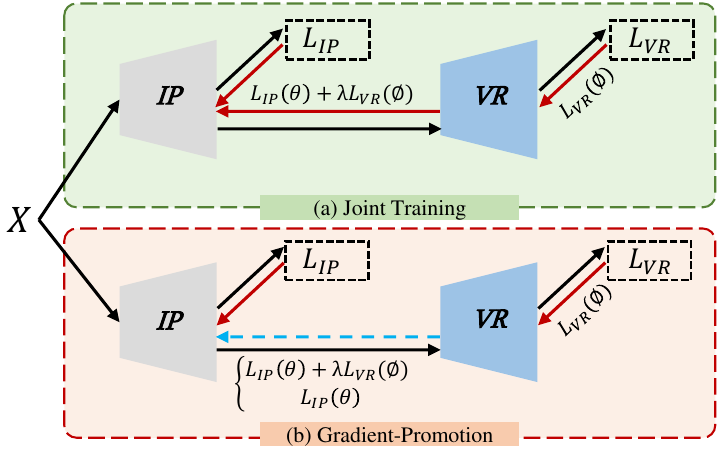}
\caption{(a) Joint training puts the upstream and downstream models together, but ignores that different models have inconsistent feature requirements. (b) Our \emph{GradProm} dynamically selects parameters that can be updated according to the similarity of different model gradients.}
\vspace{-2mm}
\label{fig2}
\end{figure} 

\subsection{Gradient-Promotion (GradProm)}
\label{sec3.2}
In a task-driven medical IQE system, we mainly focus on the performance of $IP$. The core role of $VR$ is only to assist $IP$ to be assigned more beneficial semantic information during training~\cite{liu2022exploring,xie2015task}. Based on this consensus, we explicitly divide a task-driven medical IQE system into two sub-models: the mainstream $IP$ and the auxiliary $VR$. For convenience, we formulate that parameters of $IP$ is $\theta$, and parameters of $VR$ is $\phi$. Therefore, Eq.~\ref{eq1} can be expressed as:
\begin{equation}
L_{total}(\theta_{T}) = L_{IP}(\theta) + \lambda~L_{VR}(\phi), 
\label{eq5} 
\end{equation}
where $\theta_{T}$ denotes all learnable parameters of a task-driven medical IQE system. The solo gradient vector of $IP$ and $VR$ can be expressed as $G_{IP} = \nabla_{\theta} L_{IP}(\theta)$ and $G_{VR} = \nabla_{\phi} L_{VR}(\phi)$, respectively. In joint training, as illustrated in Figure~\ref{fig2} (a), we aim at approximating the minimization of the following formulation: 
\begin{equation}
G_{T} = \left(\underbrace{\nabla_{\theta}(L_{IP}(\theta) + \lambda L_{VR}(\phi))}_{\text{for $IP$}}; \underbrace{\nabla_{\phi} L_{VR}(\phi)}_{\text{for $VR$}}\right),
\label{eq6} 
\end{equation}
where $G_{T}$ denotes the total gradient. 

In \emph{GradProm}, we hope that the task-driven medical IQE system can dynamically adjust the training objective via the model gradients. To this end, the relation between $G_{IP}$ and $G_{VR}$ can particularly be divided into the following two scenarios: \romannumeral1) $G_{IP}$ and $G_{VR}$ are aligned in the same learning direction; \romannumeral2) $G_{IP}$ and $G_{VR}$ are not aligned in the same learning direction.  
Following~\cite{zhu2022prompt,du2018adapting,yu2020gradient}, the optimization direction $s$ between $G_{IP}$ and $G_{VR}$ can be measured by their cosine similarity, \ie, $s = \textrm{cos}(G_{IP},~G_{VR})$.
For the first scenario, that means $VR$ is beneficial for the optimization of $IP$, \ie, the auxiliary visual recognition model is useful. Then we can use both $L_{IP}(\theta)$ and $L_{VR}(\phi)$ to update $G_{IP}$ of the mainstream $IP$, \ie, $G_{IP} = \nabla_{\theta}(L_{IP}(\theta) + L_{VR}(\phi))$. 
For the second scenario, that means $VR$ is not beneficial for the optimization of $IP$, \ie, the auxiliary visual recognition model is not useful. Then we just need to use $L_{IP}(\theta)$ to update $G_{IP}$ of the mainstream $IP$, \ie, $G_{IP} = \nabla_{\theta}(L_{IP}(\theta))$. 
$L_{VR}(\phi)$ of the auxiliary $VR$ is only used to updated by its own gradient $G_{VR}$, \ie, $G_{VR} = \nabla_{\phi} L_{VR}(\phi)$ in both scenarios above. Such a dynamic loss assignment mechanism can avoid optimization conflicts caused when $G_{IP}$ and $G_{VR}$ are not in the same direction. Figure~\ref{fig2} (b) illustrates this dynamic selection mechanism. 
Although there may be other more complex intermediate scenarios in practice, such as when $IP$ and $VR$ belong to the same half-space and are roughly aligned in the same direction (\eg, $s = 0.1, 0.3,$ or $0.5$ are finer-grained states), we graciously contend that delving into these specific minutiae is unwarranted. Investigating such finer-grained cases not only fails to induce a qualitative shift in model training but also augments methodological complexity, rendering the algorithm excessively redundant. Therefore, based on the aforementioned two scenarios, \emph{GradProm} can be formulated as:
\[
G_{T} = \begin{cases}
\left(\underbrace{\nabla_{\theta}(L_{IP}(\theta) + \lambda L_{VR}(\phi))}_{\text{for $IP$}}; \underbrace{\nabla_{\phi} L_{VR}(\phi)}_{\text{for $VR$}}\right), & if~s\geq 0,\\
\left(\underbrace{\nabla_{\theta}(L_{IP}(\theta))}_{\text{for $IP$}}; \underbrace{\nabla_{\phi} L_{VR}(\phi)}_{\text{for $VR$}}\right), & if~s< 0.
\end{cases} 
\]
It is worth noting that the proposed \emph{GradProm} cannot speed up the inference process or always guarantee performance improvements~\cite{sener2018multi}. Practically, it can only be used to alleviate that bad situation, \ie, when gradients of the mainstream $IP$ and the auxiliary $VR$ are inconsistent.

In our implementation, both $IP$ and $VR$ are pre-trained on the ImageNet dataset~\cite{deng2009imagenet}, which is a famous natural image dataset. Following~\cite{zhang2021self,zhang2022deep}, to avoid model collapse, we first train $IP$ on the medical image dataset for some epochs, and then cascade $VR$ for the task-driven medical IQE. Besides, since the default dataset fed into $IP$ and $VR$ are the same, there is also no domain gap problem between the two models.

\subsection{Theoretical Analysis}
\label{sec3.3}
In this section, we provide a theoretical analysis to prove that our \emph{GradProm} can converge to the local optimization.

\begin{lemma}
For the given gradient vector fields $G_{IP} = \nabla_{\theta} L_{IP} (\theta)$ and $G_{VR} = \nabla_{\phi} L_{VR}(\phi)$, GradProm can achieve the local minimum via the following update rule:
\[
\theta_{T}^{t+1}: = \theta_{T}^{t} - \alpha^t\left(G_{IP}^t + G_{VR}^t \textrm{max}(0, \textrm{cos}(G_{IP}^t, G_{VR}^t))\right),
\]
under the condition that $\alpha^t$ is as small as possible, where $t$/$t+1$ denotes the $t$-th/$(t+1)$-th training epoch for parameter update.
\end{lemma}
\begin{proof}
We start by considering the update rule for the $T$-th model parameter $\theta$ at $t$-th epoch, which is given by:
\[
\theta_{T}^{t+1} :=  \theta_{T}^{t} - \alpha^t \triangle \theta_{T}^t,
\]
where $\triangle \theta_{T}^t = G_{IP}^t + G_{VR}^t \textrm{max}\left(0, \textrm{cos}(G_{IP}^t, G_{VR}^t)\right)$.
To prove that our \emph{GradProm} can achieve the local minimum, we need to show that the update rule above ensures a non-negative inner product between $\triangle \theta_{T}^t$ and $\nabla L_{IP}^t = \nabla_{\theta} L_{IP} (\theta^t)$. Specifically, we have:
\begin{equation}\label{eqn:inner_product}
\langle \triangle \theta_{T}^t, \nabla L_{IP}^t \rangle \geq 0.
\end{equation}
We expand the inner product in Eq.~\ref{eqn:inner_product} as follows:
\begin{equation}
\label{eqn:inner_product_expanded}
\small
\begin{split}
\langle \triangle \theta_{T}^t, \nabla L_{IP}^t \rangle & = \langle G_{IP}^t + G_{VR}^t \textrm{max}\left(0, \textrm{cos}(G_{IP}^t, G_{VR}^t)\right), \nabla L_{IP}^t\rangle \\
& = \langle G_{IP}^t, \nabla L_{IP}^t\rangle + \langle G_{VR}^t \textrm{max}(0, \textrm{cos}(G_{IP}^t, G_{VR}^t)), \nabla L_{IP}^t\rangle.
\end{split}
\end{equation}
We then can make the following two observations about the terms in Eq.~\ref{eqn:inner_product_expanded}:
\begin{enumerate}
\item $\langle G_{IP}^t, \nabla L_{IP}^t\rangle = ||\nabla L_{IP}^t||^2 \geq 0$, since the norm of a gradient vector is always non-negative.
\item $\langle G_{VR}^t \textrm{max}(0, \textrm{cos}(G_{IP}^t, G_{VR}^t)), \nabla L_{IP}^t\rangle \geq 0$, since the cosine similarity between $G_{IP}^t$ and $G_{VR}^t$ is always non-negative.
\end{enumerate}
Combining the above observations, we can conclude that $\langle \triangle \theta_{T}^t, \nabla L_{IP}^t \rangle \geq 0$. 
Suppose that $\langle \triangle \theta_{T}^t, \nabla L_{IP}^t \rangle = 0$. Then, from Eq.~\ref{eqn:inner_product_expanded}, we have:
\begin{equation}\label{eqn:inner_product_zero}
\langle G_{VR}^t \textrm{max}(0, \textrm{cos}(G_{IP}^t, G_{VR}^t)), \nabla L_{IP}^t\rangle = 0.
\end{equation}
Since $\textrm{max}(0, \textrm{cos}(G_{IP}^t, G_{VR}^t)) \geq 0$, we can infer from Eq.~\ref{eqn:inner_product_zero} that either $\langle G_{VR}^t, \nabla L_{IP}^t\rangle = 0$ or $\langle \nabla L_{IP}^t, G_{VR}^t\rangle = 0$. In either case, we can conclude that $\nabla L_{IP}^t = 0$, since the cosine similarity between two non-zero vectors is positive. Therefore, we have shown that if $\langle \triangle \theta_{T}^t, \nabla L_{IP}^t \rangle = 0$, then $\nabla L_{IP}^t = 0$. This proof implies that our proposed \emph{GradProm} approach guarantees a non-negative inner product between $\triangle \theta_{T}^t$ and $\nabla L_{IP}^t$, and can achieve the local minimum of the mainstream $IP$ model, under the condition that the learning rate $\alpha$ is sufficiently small. 
\end{proof} 
The proof demonstration establishes that the update rule employed in model training guarantees a non-negative inner product between $\triangle \theta_{T}^t$ and $\nabla L_{IP}^t$. This feature, in turn, enables the proposed \emph{GradProm} to converge towards a local minimum of the mainstream $IP$ model, without being affected by the gradient of the auxiliary $VR$ model, and deviating towards the learning direction of $VR$. Thus, we theoretically demonstrated that our proposed approach could facilitate the improvement of the convergence of the mainstream $IP$ model. 

\section{Experiments}
\subsection{Experiment Setup}
\myparagraph{Datasets.} We evaluate the performance of our proposed \emph{GradProm} approach on four challenging medical image analysis datasets with different image modalities: ISIC 2018~\cite{codella2019skin}, COVID-CT~\cite{yang2020covid}, Lizard~\cite{graham2021lizard}, and CAMUS~\cite{leclerc2019deep}. The details of each dataset are as follows:

\begin{itemize}
\item \textbf{ISIC 2018~\cite{codella2019skin}:} This skin lesion dataset contains $2,594$ RGB images and corresponding pixel-level annotated ground-truth images with a resolution of $600 \times 450$. Following~\cite{azad2020attention}, we randomly split the dataset into $80\%$ training and $20\%$ testing sets.
\item \textbf{COVID-CT~\cite{yang2020covid}:} COVID-CT is a CT dataset containing $349$ positive COVID-19 CT images and $397$ negative non-COVID-19 CT images from $216$ individual patients. Following~\cite{yang2020covid,wang2023novel}, we use $600$ images for training and $146$ images for testing.
\item \textbf{Lizard~\cite{graham2021lizard}:} Lizard consists of $238$ PNG images with $6$ nuclear categories, including ``epithelial'', ``connective'', ``lymphocytes'', ``plasma'', ``neutrophils'', and ``eosinophils''. Each image has a ground-truth label, which includes a nuclear class label and a segmentation mask. Following the commonly used settings~\cite{graham2021lizard}, in our experiments, we use $180$ images for training and $58$ images for testing.
\item \textbf{CAMUS~\cite{leclerc2019deep}:} CAMUS is an ultrasound image dataset of 2D echocardiography, comprising clinical examinations from a total of $500$ patients. The training set consisted of $450$ patients, among which $366$ had good to medium image quality and $84$ had poor image quality. The testing set comprised $50$ patients, of which $40$ had good to medium image quality and $10$ had poor image quality. Besides, corresponding manual references provided by a cardiologist expert, along with additional information such as diastolic-systolic phase instants, were also made available. In our experiments, the ground-truth regions of the endocardium are utilized for semantic segmentation.
\end{itemize}
For data augmentation, as in~\cite{liu2022exploring,zhou2018unet++}, we use horizontal flip, vertical flip, and random rotation with a range of $[-10^\circ, 10^\circ]$. In addition, following the commonly used strategies~\cite{anwar2018medical,zhang2022new}, we also use center cropping with a scale of $[50\% - 100\%]$ to increase the diversity of the training data. 

\myparagraph{Baselines.}
Following ESTR~\cite{liu2022exploring}, we select SR-ResNet~\cite{ledig2017photo} with the MSE loss as the baseline model for the mainstream $IP$. For the auxiliary diagnosis $VR_{dia}$, we use ResNet~\cite{he2016deep} for image classification as the baseline model, and for the auxiliary segmentation $VR_{seg}$, we use the classical UNet~\cite{ronneberger2015u} with ResNet as the backbone. Besides, ESTR~\cite{liu2022exploring} is also employed as a baseline model in our experiments, which is trained with fixed parameters as in its paper. Empirically, $\lambda$ is set to $10^{-4}$. These choices are made based on their proven effectiveness in medical image analysis tasks and their compatibility with our proposed \emph{GradProm} strategy.

\myparagraph{Training Details.}
We implemente the entire network, including the baseline models, on the PyTorch platform~\cite{paszke2019pytorch} using four GeForce RTX 3090 GPUs. We use Adam~\cite{kingma2014adam} as the optimizer, with an initial learning rate of $0.0001$. The mini-batch size is set to $8$ for both $IR$ and $VR$ during the training stage, and other settings followed the baseline model unless otherwise instructed.

\myparagraph{Evaluation Metrics.} To evaluate the performance of \emph{GradProm}, we adopt the commonly used PSNR and SSIM as evaluation metrics. To reduce the randomness of the results, we repeat each experiment three times and report the mean value and the standard deviation in the ablation study.
For visual recognition tasks, we use accuracy and mIoU for diagnosis and segmentation, respectively. While we acknowledge that there are many other evaluation metrics, we chose these representative metrics due to the page limit and research scope of our experiments. We believe that these metrics can provide a comprehensive evaluation of the performance of our \emph{GradProm} and enable a fair comparison with existing state-of-the-art methods in the field.

\subsection{Image Pre-Processing}
\label{sec:4:2}
To simulate real-world scenarios of low-quality medical images that require IQE, we followed the approach of previous work \cite{liu2022exploring,ronneberger2015u,jaffe2017super,patil2022medical,ledig2017photo,mahapatra2019image,huang2022winnet,behjati2023single,wang2023novel} to obtain synthetic low-quality input images. Specifically, for {image denoising}, we added a certain standard deviation $\sigma$ of white Gaussian noise to the normal image by default. For the use of other types of noise, we will introduce them in the subsequent specific experimental content sections.
For {super-resolution}, we down-sampled the normal image by a factor of $\gamma$ to obtain a low-resolution input image. The learning objective of our proposed \emph{GradProm} approach is to recover these processed images to their normal state. These image pre-processing steps enable us to evaluate the effectiveness of our proposed \emph{GradProm} strategy in improving medical image quality. The synthetic low-quality images are similar to those encountered in real-world medical applications, where image quality can be compromised due to various factors such as low-dose radiation exposure or sensor limitations of the imaging equipment. By improving the quality of these low-quality images using our proposed approach, we can potentially improve the accuracy of medical diagnosis and treatment planning, thus benefitting patients and healthcare providers alike.
\begin{table*}[t]
\centering
\scriptsize
\renewcommand\arraystretch{1.3}
\setlength{\tabcolsep}{2pt}{
\caption{Ablation study results on ISIC 2018~\cite{codella2019skin} for image denoising under different noise rates. $VR$ is the ResNet-101~\cite{he2016deep} for image diagnosis (classification). $\sigma$ denotes the rate of the white Gaussian noise. ``--'' denotes that there is no such a result. ``$\uparrow$'' means the higher the better. ``Sup.'' and ``Unsup.'' denotes the supervised and the unsupervised setting of the auxiliary visual recognition model, respectively.}
\vspace{-3mm}
\label{tab1}
\begin{tabular}{lrcccccccccccc}
\midrule[1pt]
& & \multicolumn{3}{c}{Noise: $\sigma$=0.05} & \multicolumn{3}{c}{Noise: $\sigma$=0.1} & \multicolumn{3}{c}{Noise: $\sigma$=0.2} & \multicolumn{3}{c}{Noise: $\sigma$=0.3} \\
\cmidrule(r){3-5} \cmidrule(r){6-8} \cmidrule(r){9-11} \cmidrule(r){12-14} Setting & Strategy & PSNR $\uparrow$ & SSIM  $\uparrow$ & Acc $\uparrow$ & PSNR $\uparrow$ & SSIM $\uparrow$ & Acc $\uparrow$ & PSNR $\uparrow$ & SSIM $\uparrow$ & Acc $\uparrow$ & PSNR $\uparrow$ & SSIM $\uparrow$ & Acc $\uparrow$ \\
\hline \hline 
\multicolumn{2}{r}{\emph{Benchmark-\romannumeral2}~\cite{ledig2017photo}} & -- & -- & 0.803 & -- & -- & 0.787 & -- & -- & 0.712 & -- & -- & 0.626\\ 
\multicolumn{2}{r}{\emph{Benchmark-\romannumeral1}~\cite{he2016deep}} & 31.281$\pm$2.43 & 0.870$\pm$0.019 & -- & 27.735$\pm$2.22 & 0.854$\pm$0.011 & -- & 22.367$\pm$2.11 & 0.654$\pm$0.015 & -- & 15.363$\pm$1.33 & 0.523$\pm$0.020 & -- \\ 
\hline 
\multirow{3}{*}{Unsup.} & Joint training & 33.033$\pm$1.82 & 0.901$\pm$0.018 & 0.856 & 31.815$\pm$2.02 & 0.895$\pm$0.009 & 0.841 & 24.357$\pm$2.41 & 0.712$\pm$0.017 & 0.769 & 19.363$\pm$1.05 & 0.603$\pm$0.018 & 0.702 \\
& Frozen-params. & 34.172$\pm$1.23 & 0.928$\pm$0.019 & 0.872 & 32.152$\pm$1.78 & 0.906$\pm$0.007 & 0.859 & 24.926$\pm$2.13 & 0.791$\pm$0.018 & 0.788 & 20.511$\pm$1.00 & 0.630$\pm$0.017 & 0.759 \\
& \emph{GradProm} & \textbf{35.368$\pm$0.92} & \textbf{0.935$\pm$0.014} & \textbf{0.880} & \textbf{33.383$\pm$1.00} & \textbf{0.915$\pm$0.006} & \textbf{0.863} & \textbf{25.981$\pm$1.35} & \textbf{0.801$\pm$0.006} & \textbf{0.797} & \textbf{21.422$\pm$0.88} & \textbf{0.647$\pm$0.009} & \textbf{0.773} \\
\hline 
\multirow{3}{*}{Sup.} & Joint training & 35.823$\pm$2.12 & 0.943$\pm$0.018 & 0.871 & 33.474$\pm$2.11 & 0.925$\pm$0.007 & 0.862 & 25.392$\pm$2.19 & 0.790$\pm$0.012 & 0.805 & 19.989$\pm$1.01 & 0.662$\pm$0.013 & 0.751 \\
& Frozen-params. & 36.017$\pm$2.01 & 0.951$\pm$0.019 &0.892 &33.962$\pm$2.25 &0.933$\pm$0.009 & 0.869 & 25.526$\pm$1.72 & 0.799$\pm$0.011 & 0.815 & 20.361$\pm$1.00 & 0.673$\pm$0.015 & 0.779\\
& \emph{GradProm} & \textbf{36.835$\pm$1.12} & \textbf{0.963$\pm$0.013} & \textbf{0.903} & \textbf{34.646$\pm$1.11} & \textbf{0.942$\pm$0.007} & \textbf{0.874} & \textbf{26.025$\pm$1.22} & \textbf{0.808$\pm$0.010} & \textbf{0.828} & \textbf{21.353$\pm$0.67} & \textbf{0.683$\pm$0.008} & \textbf{0.796} \\
\midrule[1pt]
\end{tabular}}
\vspace{-2mm}
\end{table*}

\subsection{Comparison Methods}
We compare our proposed \emph{GradProm} approach with several baseline methods to evaluate its effectiveness in improving image quality and achieving superior performance in the task-driven IQE paradigm. Specifically, we compare \emph{GradProm} with the following methods:

\begin{itemize}
\item \textbf{\emph{Benchmark-\romannumeral1}}: This baseline method utilizes SR-ResNet~\cite{ledig2017photo} for image enhancement without using an auxiliary model. This benchmark serves as a performance lower bound for \emph{GradProm}.
\item \textbf{\emph{Benchmark-\romannumeral2}}: This baseline method utilizes a pure ResNet~\cite{he2016deep} for image diagnosis.
\item \textbf{\emph{Benchmark-\romannumeral3}}: This baseline method utilizes a pure UNet~\cite{ronneberger2015u} with ResNet as the backbone for semantic segmentation.
\item \textbf{\emph{Joint Training}}: This is a commonly used strategy for training the task-driven IQE paradigm and multi-task learning models. In this approach, the encoder is jointly trained with the auxiliary model.
\item \textbf{\emph{Frozen-params. Training}}: This is another commonly used strategy for training the task-driven IQE paradigm. In this approach, following ESTR~\cite{liu2022exploring}, we freeze the parameters of $VR$ during training, and its gradient is not passed into $IP$.
\end{itemize}
We compare performance of the above methods with \emph{GradProm} on ISIC 2018~\cite{codella2019skin} with ResNet-50~\cite{he2016deep} as the encoder network for image denoising. We also believe that using a more advanced backbone network for the IQE model would lead to better performance. The experimental results demonstrate the effectiveness of \emph{GradProm} in improving image quality and achieving superior performance in the task-driven IQE paradigm over the above baselines.
\begin{table*}[t]
\centering
\scriptsize
\renewcommand\arraystretch{1.3}
\setlength{\tabcolsep}{1pt}{
\caption{Ablation study results on ISIC 2018~\cite{codella2019skin} for image denoising under different noise rates. $VR$ is the 
UNet~\cite{ronneberger2015u} with ResNet-101~\cite{he2016deep} as backbone for semantic segmentation by default. $\sigma$ denotes the rate of the white Gaussian noise. ``--'' denotes that there is no such a result. ``$\uparrow$'' means the higher the better. ``Sup.'' and ``Unsup.'' denotes the supervised and the unsupervised setting of the auxiliary visual recognition model, respectively. ``Multi-Task'' denotes that both  image diagnosis and semantic segmentation models are used in the VR task.}
\vspace{-3mm}
\label{tab2}
\begin{tabular}{lrcccccccccccc}
\midrule[1pt]
& & \multicolumn{3}{c}{Noise: $\sigma$=0.05} & \multicolumn{3}{c}{Noise: $\sigma$=0.1} & \multicolumn{3}{c}{Noise: $\sigma$=0.2} & \multicolumn{3}{c}{Noise: $\sigma$=0.3} \\
\cmidrule(r){3-5} \cmidrule(r){6-8} \cmidrule(r){9-11} \cmidrule(r){12-14} Setting & Strategy & PSNR $\uparrow$ & SSIM  $\uparrow$ & mIoU $\uparrow$ & PSNR $\uparrow$ & SSIM $\uparrow$ & mIoU $\uparrow$ & PSNR $\uparrow$ & SSIM $\uparrow$ & mIoU $\uparrow$ & PSNR $\uparrow$ & SSIM $\uparrow$ & mIoU $\uparrow$ \\
\hline \hline 
\multicolumn{2}{r}{\emph{Benchmark-\romannumeral3}~\cite{ledig2017photo}} & -- & -- & 0.767 & -- & -- & 0.720 & -- & -- & 0.631 & -- & -- & 0.582\\ 
\multicolumn{2}{r}{\emph{Benchmark-\romannumeral1}~\cite{he2016deep}} & 31.281$\pm$2.11 & 0.870$\pm$0.015 & -- & 27.735$\pm$2.10 & 0.854$\pm$0.013 & -- & 22.367$\pm$2.55 & 0.654$\pm$0.014 & -- & 15.363$\pm$2.03 & 0.523$\pm$0.015 & -- \\ 
\hline 
\multirow{3}{*}{Unsup.} & Joint training & 32.273$\pm$2.01 & 0.895$\pm$0.011 & 0.824 & 31.802$\pm$2.11 & 0.878$\pm$0.011 & 0.802 & 23.835$\pm$2.22 & 0.716$\pm$0.013 & 0.706 & 19.025$\pm$1.77 & 0.584$\pm$0.010 & 0.688 \\
& Frozen-params. & 32.826$\pm$2.00 & 0.899$\pm$0.009 & 0.853 & 32.038$\pm$2.02 & 0.898$\pm$0.010 & 0.852 & 24.025$\pm$1.90 & 0.753$\pm$0.011 & 0.737 & 19.836$\pm$1.50 & 0.601$\pm$0.008 & 0.709 \\
& \emph{GradProm} & \textbf{33.525$\pm$1.05} & \textbf{0.916$\pm$0.004} & \textbf{0.887} & \textbf{33.270$\pm$1.51} & \textbf{0.914$\pm$0.007} & \textbf{0.861} & \textbf{26.373$\pm$1.11} & \textbf{0.813$\pm$0.007} & \textbf{0.781} & \textbf{21.262$\pm$1.18} & \textbf{0.667$\pm$0.005} & \textbf{0.755} \\
\hline 
\multirow{3}{*}{Sup.} & Joint training & 34.163$\pm$2.01 & 0.903$\pm$0.011 & 0.857 & 32.703$\pm$2.07 & 0.904$\pm$0.012 & 0.841 & 24.734$\pm$2.34 & 0.752$\pm$0.011 & 0.755 & 19.462$\pm$1.60 & 0.606$\pm$0.011 & 0.703 \\
& Frozen-params. & 35.282$\pm$1.89 & 0.923$\pm$0.007 & 0.862 & 33.029$\pm$1.91 & 0.945$\pm$0.011 & 0.881 & 25.027$\pm$2.10 & 0.804$\pm$0.010 & 0.811 & 19.912$\pm$1.55 & 0.681$\pm$0.008 & 0.733 \\
& \emph{GradProm} & \textbf{36.825$\pm$1.33} & \textbf{0.952$\pm$0.005} & \textbf{0.881} & \textbf{34.820$\pm$1.25} & \textbf{0.962$\pm$0.006} & \textbf{0.895} & \textbf{26.901$\pm$1.03} & \textbf{0.894$\pm$0.004} & \textbf{0.845} & \textbf{21.261$\pm$1.01}& \textbf{0.732$\pm$0.003} & \textbf{0.761 }\\
\hline \hline 
Unsup. & Multi-Task & 31.826 $\pm$2.12	& 0.906 $\pm$0.011 & 0.809 & 30.272 $\pm$1.05 & 0.891 $\pm$0.005 & 0.791 & 23.289 $\pm$1.62 & 0.755 $\pm$0.002 & 0.790 & 19.850 $\pm$0.96 & 0.611 $\pm$0.005 & 0.733 \\
Sup. & Multi-Task & 32.016 $\pm$1.67 & 0.912 $\pm$0.010 & 0.812 & 31.026$\pm$0.91	& 0.909$\pm$0.007 &	
0.800 &	24.328 $\pm$0.77 & 0.787$\pm$0.004 & 0.811 & 19.902 $\pm$0.88 & 0.652 $\pm$0.008 & 0.737 \\
\midrule[1pt]
\end{tabular}}
\vspace{-2mm}
\end{table*}

\subsection{Ablation Study}
In this section, we conduct ablation studies on the representative ISIC 2018~\cite{codella2019skin} with a ResNet-50~\cite{he2016deep} as the encoder for image denoising to answer the following five questions:
\begin{itemize}
\item \textbf{\emph{Q-1}}) What is the impact of \emph{GradProm} on model performance under different noise rates?
\item \textbf{\emph{Q-2}}) How does \emph{GradProm} perform compared to other training strategies of the task-driven IQE paradigm?
\item \textbf{\emph{Q-3}}) How does GradProm perform under different training supervisions?
\item \textbf{\emph{Q-4}}) How does GradProm perform with multiple visual tasks?
\item \textbf{\emph{Q-5}} How does GradProm perform in more challenging and novel scenarios?
\end{itemize}

\myparagraph{\emph{A-1}) \emph{GradProm} performance under different noise rates.} To implement this experiment, we add different levels of white Gaussian noise to the normal images, ranging from $\sigma = [0.05, 0.1, 0.2, 0.3]$, and evaluate experimental results on $VR_{dia}$ and $VR_{seg}$, as shown in Table~\ref{tab1} and Table~\ref{tab2}, respectively, for both supervised and unsupervised settings. Compared to the \emph{Benchmarks}, joint training, and Frozen-params., we can observe that our proposed \emph{GradProm} can consistently improve the performance of the baseline model (\ie, ESTR~\cite{liu2022exploring} with Frozen-params.) at different noise rates. 
For instance, under the setting that image diagnosis is used as the auxiliary task in Table~\ref{tab1}, \emph{GradProm} improves ESTR~\cite{liu2022exploring} with Frozen-params. by up to $1.196/0.007$, $1.231/0.009$, $1.055/0.010$, and $0.911/0.017$ PSNR/SSIM for noise rates of $\sigma = 0.05, 0.1, 0.2,$ and $0.3$ on average, respectively. 
At the same time, \emph{GradProm} also outperforms the baseline ESTR~\cite{liu2022exploring} under the setting that UNet~\cite{ronneberger2015u} with ResNet-101~\cite{he2016deep} is used as as auxiliary backbone for semantic segmentation. \emph{GradProm} improves ESTR by $5.544/0.082$, $7.085/0.108$, $4.534/0.240$, and $5.898/0.197$ PSNR/SSIM under the noise rates of $\sigma = 0.05, 0.1, 0.2,$ and $0.3$ on average, respectively. The experimental results, conducted under settings with image diagnosis and semantic segmentation as auxiliary tasks, demonstrate that our method consistently boosts the baseline model across various noise rates. This evidence highlights the robust capacity of the proposed \emph{GradProm} in managing images with differing noise levels.
\begin{figure*}[t]
\centering
\includegraphics[width=.99\textwidth]{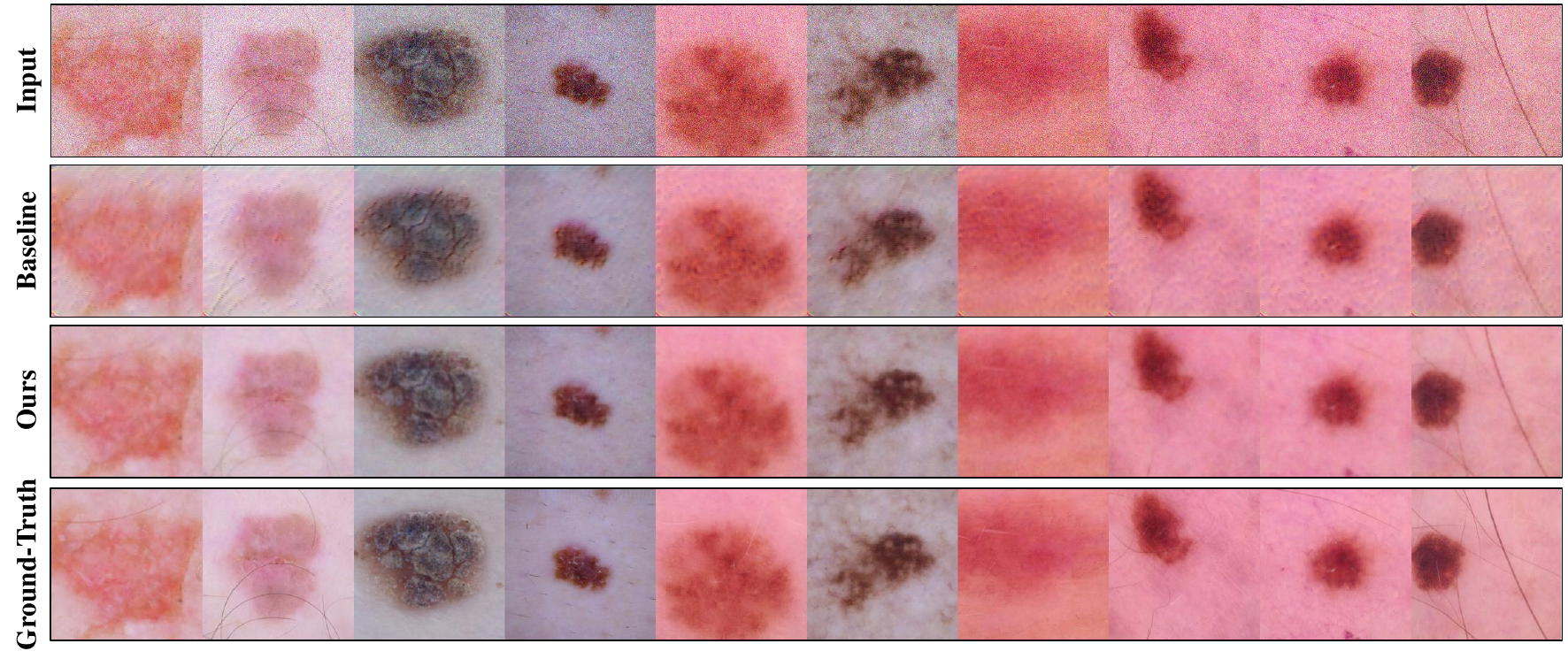}
\vspace{-5mm}
\caption{Qualitative result comparisons with the baseline ESTR~\cite{liu2022exploring} on ISIC 2018~\cite{codella2019skin} with ResNet-50~\cite{he2016deep} as the encoder network for image denoising, where $\sigma$ is set to $0.3$ and $VR_{dia}$ is used as the visual recognition model.}
\vspace{-2mm}
\label{fig3}
\end{figure*} 
\begin{figure*}[t]
\centering
\includegraphics[width=.85\textwidth]{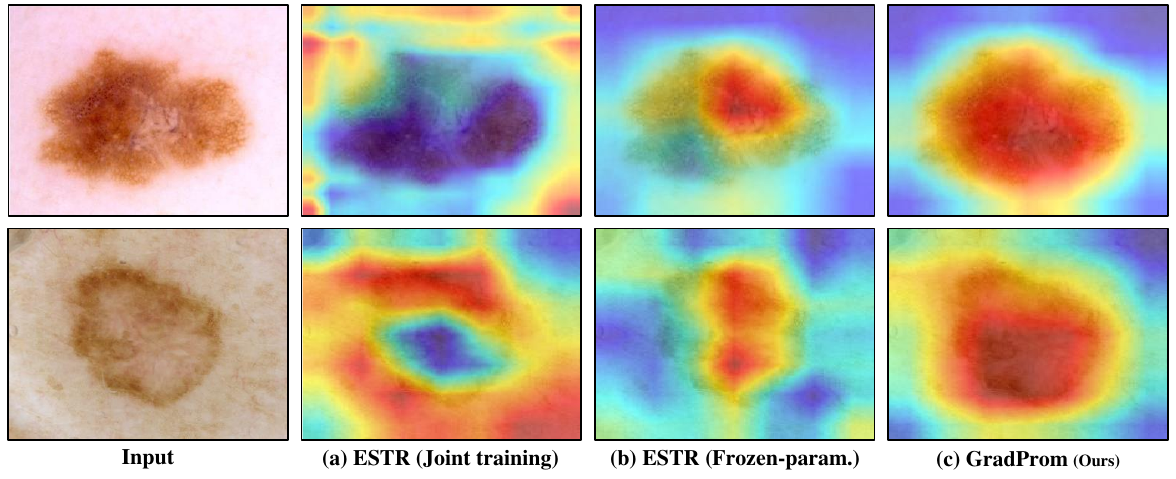}
\vspace{-5mm}
\caption{Class activation maps (CAMs) result comparisons with the baseline model ESTR~\cite{liu2022exploring} on ISIC 2018~\cite{codella2019skin} with ResNet-50~\cite{he2016deep} as the encoder for image denoising, where $\sigma$ is set to $0.3$ and $VR_{dia}$ is used as the visual recognition model.}
\vspace{-4mm}
\label{fig32}
\end{figure*} 

\myparagraph{\emph{A-2}) Superiority of \emph{GradProm} over other training strategies.} We compare \emph{GradProm} with the commonly used joint training~\cite{zeng2019joint} and Frozen-params.~\cite{baslamisli2018joint} under the same baseline model (\ie, ESTR~\cite{liu2022exploring}). In Tables~\ref{tab1} and~\ref{tab2}, we separately select image diagnosis and semantic segmentation as the auxiliary tasks. Our results demonstrate that \emph{GradProm} outperforms both joint training and Frozen-params. by a noticeable margin under the support of both $VR_{dia}$ and $VR_{seg}$. For example, in Table~\ref{tab2}, \emph{GradProm} outperforms the joint training strategy with the unsupervised training manner by $4.475/0.057$, $3.018/0.084$, $3.066/0.178$, and $2.236/0.148$ PSNR/SSIM under the noise rates of $\sigma = 0.05, 0.1, 0.2,$ and $0.3$, respectively. 
Besides, we also observe that with the pixel-level ground-truth of the skin lesion dataset, under the assistance of $VR_{seg}$, \emph{GradProm} attains a more significant performance gains compared to using $VR_{dia}$. This suggests that for datasets with a large proportion of foreground regions with  pixel-level ground-truth in the images, selecting a semantic segmentation model yields more pronounced benefits for the IQE model than opting for an image diagnosis model. These results not only validate the generality of our approach on different task combinations but also demonstrate that our proposed \emph{GradProm} can solve the problem of difficult training when two models have large differences in feature requirements.
\begin{table}[t]
\centering
\small
\renewcommand\arraystretch{1.1}
\setlength{\tabcolsep}{.1pt}{
\caption{Ablation study results on ISIC 2018~\cite{codella2019skin} under challenging and novel scenarios. $VR$ is the 
UNet~\cite{ronneberger2015u} with ResNet-101~\cite{he2016deep} as backbone for semantic segmentation. ``--'' denotes that there is no such a result.}
\vspace{-3mm}
\label{taba4}
\begin{tabular}{rcccccc}
\midrule[1pt]
~ & \multicolumn{3}{c}{Unsup.} & \multicolumn{3}{c}{Sup.} \\
& PSNR $\uparrow$ & SSIM $\uparrow$ & mIoU $\uparrow$ & PSNR $\uparrow$ & SSIM $\uparrow$ & mIoU $\uparrow$ \\
\hline \hline 
\emph{Benchmark-\romannumeral3}~\cite{ledig2017photo} &-- & -- & 0.285 &-- & -- & 0.339\\ 
\emph{Benchmark-\romannumeral1}~\cite{he2016deep} & 9.272 & 0.362 & -- &  11.722 & 0.400 & -- \\ 
\hline 
Joint training & 9.522 & 0.390 & 0.303 & 9.821 & 0.399 & 0.408 \\
Frozen-params. & 9.967 & 0.411 & 0.414 & 9.980 & 0.418 & 0.453\\
\emph{GradProm} & \textbf{10.351} & \textbf{0.499} & \textbf{0.501} & \textbf{11.882} & \textbf{0.551} & \textbf{0.597}\\
\midrule[1pt]
\end{tabular}}
\vspace{-4mm}
\end{table}

\myparagraph{\emph{A-3}) Effectiveness of \emph{GradProm} with different supervisions.} 
In section~\ref{sec3.1}, we introduced two fundamental task-driven IQE models for $VR$, specifically the unsupervised and supervised settings. In this ablation study, we explore the effectiveness of our proposed \emph{GradProm} approach under both supervised and unsupervised settings. The experimental results for image diagnosis and semantic segmentation as auxiliary tasks are presented in Table~\ref{tab1} and Table~\ref{tab2}, respectively. Overall, we can observe that our proposed \emph{GradProm} approach improves the performance of the baseline ESTR~\cite{liu2022exploring} model and the \emph{Benchmarks} under both supervisions. 
Notably, we observed that \emph{GradProm} performed better in the unsupervised scenario with $VR_{dia}$, while under $VR_{seg}$, it roughly performed better in the supervised scenario. This difference may be due to the specific feature requirements of the vision recognition model. For instance, when both the mainstream model (\ie, quality enhancement) and auxiliary model (\ie, segmentation) are pixel-level tasks, fine-grained supervision is necessary to ensure that the model does not deviate from the intended learning objectives~\cite{baslamisli2018joint}. 
Besides, these experimental results in Table~\ref{tab1} and Table~\ref{tab2} also provide guidance for selecting auxiliary task settings in the future. If the auxiliary task has pixel-level labels for the semantic segmentation task, it is necessary to configure the mainstream IQE model as a supervised learning manner to achieve better experimental performance. If the auxiliary task only has image-level class labels for diagnosis (classification), the mainstream IQE model should be set as an unsupervised learning manner.

\myparagraph{\emph{A-4}) Effectiveness of \emph{GradProm} with multiple visual tasks.} Considering the fairness of result comparisons and the availability of labels in the used dataset, we report experimental results of performing both image diagnosis~\cite{he2016deep} and semantic segmentation\cite{ronneberger2015u}in the VR task. We adopt the gradient update strategy proposed in our work as long as the IP gradient is aligned with those of image diagnosis or semantic segmentation model. The obtained results are shown in the bottom half of Table~\ref{tab2}. 
We can observe that simultaneously using these two models in the VR task does not improve the performance. Instead, there is a decrease in accuracy compared to using these two models independently. This observation can be attributed to the inconsistent feature requirements between image diagnosis and semantic segmentation models, as illustrated in (c) and (d) of Figure~\ref{figr1}. Therefore, simultaneous training of these two models with different feature requirements may not lead to optimal performance unless these models have similar or approximate feature requirements, such as semantic segmentation and saliency detection, as in~\cite{zeng2019joint}.

\myparagraph{\emph{A-5}) Effectiveness of \emph{GradProm} under more challenging and novel scenarios.} In section~\ref{sec:4:2}, we added plain white Gaussian noise to the normal image. In this ablation study, we investigate the performance of \emph{GradProm} in more challenging and novel scenarios. To conduct this experiment, we further introduce Poisson noise and Gaussian blur on top of the previously added white Gaussian noise ($\sigma = 0.3$), resulting in a more complex situation. This combination of noise and blur is typically observed during the collection of clinical medical imaging datasets~\cite{ma2013dictionary,huang2022winnet}. Following the commonly used setting~\cite{patil2022medical,zhang2022new}, the noise rate of Poisson noise is set to $0.1$, and the Gaussian blur is set with a kernel size of $3$ and a standard deviation of $2.0$.  The experimental results for both the supervised and unsupervised settings are presented in Table~\ref{taba4}. We observe that under the influence of multiple types of noise and blur, the performance of all models is significantly lower compared to the performance under a single type of noise (in Table~\ref{tab1} and Table~\ref{tab2}). This indicates that excessive noise and low image quality do have a significant negative impact on the recognition performance.
Besides, we observe that our \emph{GradProm} can still improve the performance of the baseline model in this challenging scenario. For example, our method achieves a PSNR/SSIM improvement of $0.384/0.088$ and $0.133/0.144$ compared to ESTR~\cite{liu2022exploring} under the Frozen-params. training strategy, in the unsupervised and supervised settings, respectively. We also observe that compared to the experimental results in Table~\ref{tab1} and Table~\ref{tab2}, the performance gain of \emph{GradProm} in this challenging scenario is slightly reduced. This suggests that the learning-based model can only learn limited cues for quality enhancement on severely low-quality images, and therefore may only improve image quality within a certain range, but cannot handle completely collapsed images. This conclusion is consistent with most IQE methods~\cite{johnson2016perceptual,ledig2017photo,liu2022exploring,mahapatra2019image,patil2022medical,huang2022winnet,ma2013dictionary,zhang2022new}.

\myparagraph{Qualitative results.}
We compare the qualitative results of our proposed method with the baseline~\cite{liu2022exploring} in Figure~\ref{fig3}, where $\sigma$ is set to $0.3$, and $VR_{dia}$ is used as the visual recognition model. It can be observed that our \emph{GradProm} achieves visually better denoising performance than the baseline model. Specifically, our approach produces processed images that are closer to the ground-truth images in terms of image smoothing and preserving the integrity of small foreground objects (\eg, the ``hair'' and ``small lesion''). Moreover, our approach demonstrates a significant advantage in the background regions. Qualitative results of the baseline model~\cite{liu2022exploring} in background regions appear more over-textured, while our results are closer to reality. This finding indicates that \emph{GradProm} not only achieves superior quantitative results but also better qualitative results. Furthermore, it demonstrates that the auxiliary visual recognition in the task-driven IQE system can improve the learning of the mainstream mainstream image enhancement.

We also visualize some class activation maps (CAMs) of the auxiliary $VR_{dia}$ in Figure~\ref{fig32}. CAMs can be used to indicate which regions of an image the model is focusing on. We can observe that compared to joint training and Frozen-params. on the same baseline ESTR~\cite{liu2022exploring}, the CAMs of \emph{GradProm} focus more on the foreground object regions. Conversely, the CAMs of ESTR with joint training and Frozen-params. focus on regions outside the object or part of the object, respectively. These results validate that high-quality image enhancement results can benefit the downstream auxiliary visual recognition model. Moreover, they also validate that our \emph{GradProm} does not cause the auxiliary task to deviate from the desired direction that needs to be optimized during training, achieving a win-win performance. Thus, the two sets of visualization results presented above confirm the effectiveness of our proposed GradProm in enhancing image quality while preserving important visual details, and demonstrate the superiority of the proposed task-driven IQE system. These qualitative results further demonstrate that our method not only produces visually superior images but also yields images that are advantageous for machine learning models to perform visual recognition tasks.

\subsection{Exploration on Domain Generalization}
As the aforementioned experimental results demonstrate that our proposed \emph{GradProm} can overcome the bias caused by the auxiliary $VR$ model, we investigate its effectiveness on domain generalization in this section. In clinical practice, sometimes we need to test and train on different datasets due to the need to protect patient privacy, \ie, training on one dataset but validating on another related dataset~\cite{chen2023federated}. Therefore, it is necessary for the model to have strong domain generalization capacity across datasets. To validate our method's capacity in this aspect, we train on ISIC 2018~\cite{codella2019skin} and test on Lizard~\cite{graham2021lizard}. Image diagnosis is used as the auxiliary task, where $\sigma$ is set to 0.3. The experimental results are given in Table~\ref{taba5}, where the experimental results of the ``Base'' row indicate that the training and testing images belong to the same dataset, \ie, there is no domain gap problem in these results. Firstly, we observe that compared to the scenario where training and testing images are within the same dataset, both ESTR and \emph{GradProm} exhibit relatively lower performance, indicating the presence of a domain gap problem between the two datasets. Then, we note that in comparison to ESTR~\cite{liu2022exploring}, our \emph{GradProm} attains superior IQE across datasets, signifying the domain generalization effectiveness of our training strategy. Furthermore, our \emph{GradProm} consistently delivers performance improvements under both unsupervised and supervised experimental settings. We finally achieve $13.273/0.325$ and $13.825/0.458$ PSNR/SSIM on these two settings. 
\begin{table}[t]
\centering
\small
\renewcommand\arraystretch{1.1}
\setlength{\tabcolsep}{5.5pt}{
\vspace{-3mm}
\caption{Experimrntal results on model generalization for image denoising. ``$\uparrow$'' means the higher the better.The experimental results of the ``Base'' row indicate that in this set of experiments, the training and testing data belong to the same dataset.}
\label{taba5}
\begin{tabular}{rcccc}
\midrule[1pt]
& \multicolumn{2}{c}{Unsup.} & \multicolumn{2}{c}{Sup.} \\
& PSNR $\uparrow$ & SSIM $\uparrow$ & PSNR $\uparrow$ & SSIM $\uparrow$ \\
\hline \hline 
Base ESTR~\cite{liu2022exploring} & 15.326 & 0.427  & 16.352 & 0.501 \\ 
Base \emph{GradProm} & 16.736 & 0.499  & 17.051 & 0.611 \\ 
\hline 
ESTR~\cite{liu2022exploring} & 8.236 & 0.285 & 8.262 & 0.292 \\
\emph{GradProm} & \textbf{13.273} & \textbf{0.325} & \textbf{13.825} & \textbf{0.458} \\
\midrule[1pt]
\end{tabular}}
\vspace{-4mm}
\end{table}
\begin{table*}[t]
\centering
\scriptsize
\renewcommand\arraystretch{1.3}
\setlength{\tabcolsep}{9pt}{
\caption{Result comparisons with state-of-the-arts on ISIC 2018~\cite{codella2019skin}for image denoising on different noise rates. ``Task-Driven?'' means if this method is a task-driven model. $VR_{dia}$ and $VR_{seg}$ denote that results are obtained by image diagnosis and semantic segmentation, respectively.}
\vspace{-4mm}
\label{tab3}
\begin{tabular}{rrrccccccc}
\midrule[1pt]
&  & & & \multicolumn{2}{c}{$\sigma$=0.1} & \multicolumn{2}{c}{$\sigma$=0.2} & \multicolumn{2}{c}{$\sigma$=0.3} \\
\cmidrule(r){5-6} \cmidrule(r){7-8} \cmidrule(r){9-10}Methods & Backbone & Publication & Task-Driven? & PSNR $\uparrow$ & SSIM $\uparrow$ & PSNR $\uparrow$ & SSIM $\uparrow$ & PSNR $\uparrow$ & SSIM $\uparrow$ \\
\hline \hline
SR-ResNet~\cite{ledig2017photo} & VGG-22 & CVPR'17 & No & 19.478 & 0.526 & 17.926 & 0.401 & 14.273 & 0.334 \\
ESTR~\cite{liu2022exploring} & ResNet-50 & TPAMI'22 & Yes & 32.172 & 0.882 & 23.524 & 0.725 & 18.587 & 0.528 \\
ESTR~\cite{liu2022exploring} & ResNet-101 & TPAMI'22 & Yes & 33.723 & 0.899 & 25.925 & 0.790 & 20.163 & 0.667 \\
U-SAID~\cite{wang2019segmentation} & CDnCNN-B & arXiv'19 & Yes & 33.835 & 0.861 & 26.003 & 0.752 & 20.836 & 0.670 \\
D2SM~\cite{mei2022deep} & SADNet & ECCV'22 & Yes & 34.282 & 0.892 & 25.989 & 0.766 & 20.989 & 0.656\\
WinNet~\cite{huang2022winnet} & ResNet-101 & TIP'22 & No & 33.929 & 0.882 & 24.262 & 0.800 & 19.272 & 0.680\\
ADAP~\cite{jiang2022deep} & ResNet-101 & TCSVT'22 & No & 34.858 & 0.889 & 24.926 & 0.810 & 20.373 & 0.695\\
\hline
Ours(${VR_{dia}}$) & ResNet-50 & submission & Yes & 34.646 & 0.942 & 26.025 & 0.808 & 21.353 & 0.683 \\
Ours(${VR_{seg}}$) & ResNet-50 & submission & Yes & 34.820 & 0.962 & 26.901 & 0.894 & 21.261 & 0.732 \\
Ours(${VR_{dia}}$) & ResNet-101 & submission & Yes & 35.025 & 0.956 & 26.385 & 0.833 & 21.924 & 0.719 \\
Ours(${VR_{seg}}$) & ResNet-101 & submission & Yes & 36.173 & 0.971 & 28.024 & 0.901 & 23.703 & 0.761 \\
\midrule[1pt]
\end{tabular}}
\vspace{-4mm}
\end{table*}

\subsection{Comparisons with State-of-the-Arts}
In this section, we compare \emph{GradProm} with state-of-the-art methods on three datasets, including ISIC 2018~\cite{codella2019skin}, COVID-CT~\cite{yang2020covid}, Lizard~\cite{graham2021lizard}, and CAMUS~\cite{leclerc2019deep}.
\begin{table*}[t]
\centering
\scriptsize
\renewcommand\arraystretch{1.3}
\setlength{\tabcolsep}{9pt}{
\caption{Result comparisons with state-of-the-arts on COVID-CT~\cite{yang2020covid} for image denoising on different noise rates. ``Task-Driven?'' means if this method is a task-driven model. $VR_{dia}$ denote that results are obtained by image diagnosis.}
\vspace{-4mm}
\label{tab4}
\begin{tabular}{rrrccccccc}
\midrule[1pt]
& & & & \multicolumn{2}{c}{$\sigma$=0.1} & \multicolumn{2}{c}{$\sigma$=0.2} & \multicolumn{2}{c}{$\sigma$=0.3} \\
\cmidrule(r){5-6} \cmidrule(r){7-8} \cmidrule(r){9-10}Methods & Backbone & Publication & Task-Driven? & PSNR $\uparrow$ & SSIM $\uparrow$ & PSNR $\uparrow$ & SSIM $\uparrow$ & PSNR $\uparrow$ & SSIM $\uparrow$ \\
\hline \hline
SR-ResNet~\cite{ledig2017photo} & VGG-22 & CVPR'17 & No & 15.262 & 0.401 & 14.272 & 0.382 & 12.272 & 0.262\\
ESTR~\cite{liu2022exploring} & ResNet-50 & TPAMI'22 & Yes & 23.363 & 0.652 & 20.017 & 0.592 & 17.937 & 0.445 \\
ESTR~\cite{liu2022exploring} & ResNet-101 & TPAMI'22 & Yes & 24.784 & 0.692 & 22.011 & 0.612 & 19.002 & 0.522 \\
NAFSSR~\cite{chu2022nafssr}  & NAFNet & CVPR'22 & Yes & 24.635 & 0.701 & 21.638 & 0.641 & 19.635 & 0.535 \\
RLFNet~\cite{kong2022residual} & ResNet-101 & CVPR'22 & No & 25.010 & 0.732 & 22.272 & 0.688 & 20.826 & 0.544\\
DVANet~\cite{behjati2023single} & ResNet-101 & PR'23 & No & 25.252 & 0.791 & 23.075 & 0.701 & 21.088 & 0.686\\
\hline
Ours(${VR_{dia}}$) & ResNet-50 & submission & Yes & 25.115 & 0.752 & 23.346 & 0.753 & 20.631 & 0.629 \\
Ours(${VR_{dia}}$) & ResNet-101 & submission & Yes & 28.115 & 0.854 & 24.152 & 0.791 & 21.362 & 0.701 \\
\midrule[1pt]
\end{tabular}}
\vspace{-4mm}
\end{table*}
\begin{figure*}[t]
\centering
\includegraphics[width=.95\textwidth]{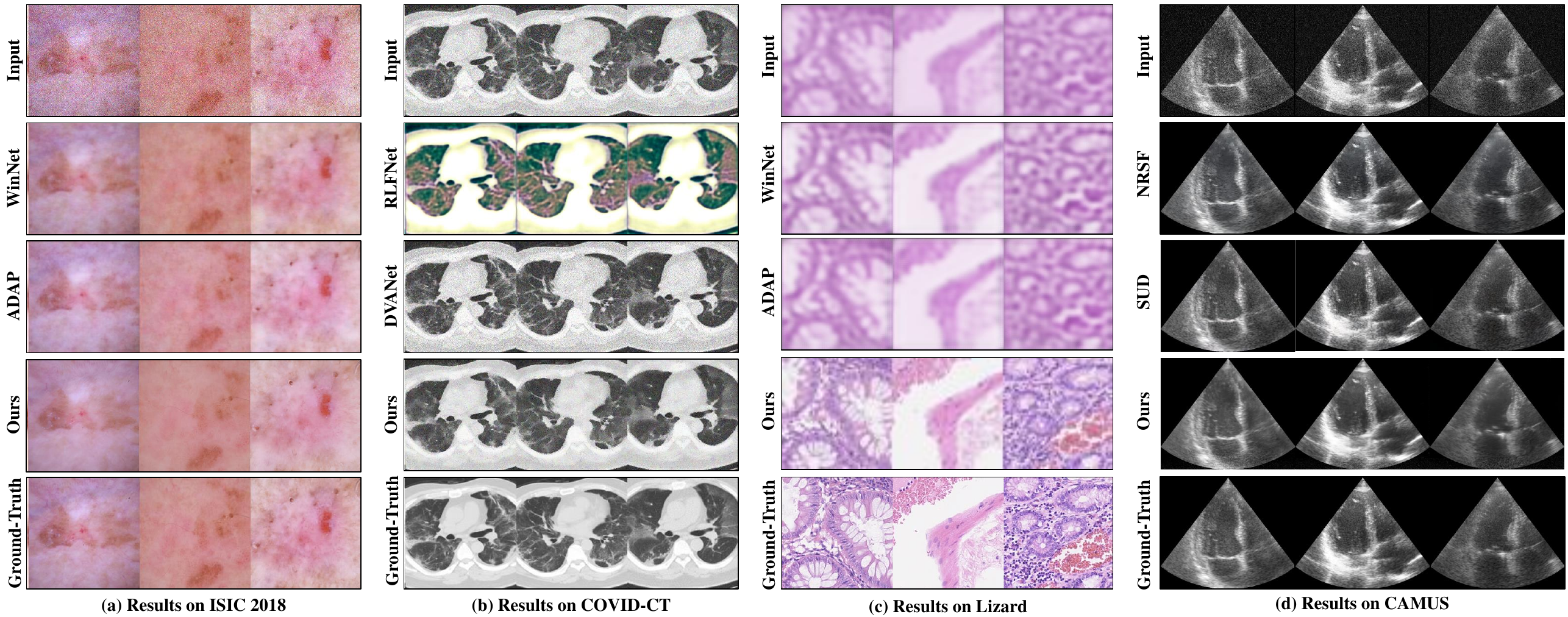}
\vspace{-4mm}
\caption{Visualization result comparisons with state-of-the-art methods on ISIC 2018~\cite{codella2019skin}, COVID-CT~\cite{yang2020covid}, Lizard~\cite{graham2021lizard}, and  CAMUS~\cite{leclerc2019deep}, where ISIC 2018, COVID-CT, and CAMUS are used for the denoising task and Lizard is used for the super-resolution task.}
\label{fig4}
\vspace{-3mm}
\end{figure*} 
\begin{table*}[t]
\centering
\scriptsize
\renewcommand\arraystretch{1.3}
\setlength{\tabcolsep}{15pt}{
\caption{Result comparisons with state-of-the-arts on Lizard~\cite{graham2021lizard} for image super-resolution on different downsampling rates. ``Task-Driven?'' means if this method is a task-driven model. $VR_{dia}$ and $VR_{seg}$ denote that results are obtained by image diagnosis and semantic segmentation, respectively.}
\vspace{-4mm}
\label{tab5}
\begin{tabular}{rrrccccc}
\midrule[1pt]
& & & & \multicolumn{2}{c}{$\gamma$=2} & \multicolumn{2}{c}{$\gamma$=4} \\
\cmidrule(r){5-6} \cmidrule(r){7-8} Methods & Backbone & Publication & Task-Driven? & SSIM $\uparrow$ & SSIM $\uparrow$ & PSNR $\uparrow$ & SSIM $\uparrow$\\
\hline \hline
SR-ResNet~\cite{ledig2017photo} & VGG-22 & CVPR'17 & No & 18.282 & 0.554 & 16.222 & 0.494\\
ESTR~\cite{liu2022exploring} & ResNet-50 & TPAMI'22 & Yes & 21.018 & 0.601 & 16.828 & 0.494 \\
ESTR~\cite{liu2022exploring} & ResNet-101 & TPAMI'22 & Yes & 23.373 & 0.621 & 17.252 & 0.551\\
D2SM~\cite{mei2022deep} & SADNet & ECCV'22 & Yes & 24.252 & 0.656 & 17.171 & 0.566\\
WinNet~\cite{huang2022winnet} & ResNet-101 & TIP'22 & No & 22.989 & 0.622 & 18.655 & 0.543\\
ADAP~\cite{jiang2022deep} & ResNet-101 & TCSVT'22 & No & 23.212 & 0.602 & 17.905 & 0.701\\
\hline
Ours(${VR_{dia}}$) & ResNet-50 & submission & Yes & 22.436 & 0.681 & 18.253 & 0.524 \\
Ours(${VR_{seg}}$) & ResNet-50 & submission & Yes & 21.937 & 0.650 & 18.724 & 0.509 \\
Ours(${VR_{dia}}$) & ResNet-101 & submission & Yes & 25.167 & 0.708 & 19.524 & 0.598 \\
Ours(${VR_{seg}}$) & ResNet-101 & submission & Yes & 23.243 & 0.685 & 19.173 & 0.566 \\
\midrule[1pt]
\end{tabular}}
\vspace{-5mm}
\end{table*}
\begin{table*}[t]
\centering
\scriptsize
\renewcommand\arraystretch{1.3}
\setlength{\tabcolsep}{9pt}{
\caption{Result comparisons with state-of-the-arts on CAMUS~\cite{leclerc2019deep} for image denoising on different noise rates. ``Task-Driven?'' means if this method is a task-driven model. $VR_{seg}$ denote that results are obtained by semantic segmentation.}
\vspace{-2mm}
\label{tab6}
\begin{tabular}{rrrccccccc}
\midrule[1pt]
& & & & \multicolumn{2}{c}{$\sigma$=0.1} & \multicolumn{2}{c}{$\sigma$=0.2} & \multicolumn{2}{c}{$\sigma$=0.3} \\
\cmidrule(r){5-6} \cmidrule(r){7-8} \cmidrule(r){9-10}Methods & Backbone & Publication & Task-Driven? & PSNR $\uparrow$ & SSIM $\uparrow$ & PSNR $\uparrow$ & SSIM $\uparrow$ & PSNR $\uparrow$ & SSIM $\uparrow$ \\
\hline \hline
SR-ResNet~\cite{ledig2017photo} & VGG-22 & CVPR'17 & No & 13.825 & 0.425 & 13.128 & 0.401 & 12.429 & 0.355\\
ESTR~\cite{liu2022exploring} & ResNet-50 & TPAMI'22 & Yes & 21.463 & 0.550 & 20.935 & 0.509 & 18.572 & 0.455 \\
DDM~\cite{stevens2023dehazing} & ResNet-50 & arXiv'23 & No & 22.163 & 0.499 & 21.881 & 0.414 & 29.767 & 0.405 \\
ESTR~\cite{liu2022exploring} & ResNet-101 & TPAMI'22 & Yes & 23.261 & 0.593 & 21.362 & 0.556 & 19.363 & 0.482 \\
NRSF~\cite{wu2022semi} & ResNet-101 & MedIA'22 & No & 24.261 & 0.599 & 22.016 & 0.534 & 20.952 & 0.479 \\
SUD~\cite{young2023supervision} & ResNet-101 & TPAMI'23 & No & 23.268 & 0.592 & 21.852 & 0.561 & 20.881 & 0.491\\
\hline
Ours(${VR_{seg}}$) & ResNet-50 & submission & Yes & 23.928 & 0.666 & 22.909 & 0.628 & 20.858 & 0.587 \\
Ours(${VR_{seg}}$) & ResNet-101 & submission & Yes & 24.425 & 0.689 & 23.022 & 0.700 & 21.055 & 0.599\\
\midrule[1pt]
\end{tabular}}
\vspace{-2mm}
\end{table*}

\myparagraph{Results on ISIC 2018.} We compare image denoising results on ISIC 2018 with state-of-the-art methods, including SR-ResNet~\cite{ledig2017photo}, ESTR~\cite{liu2022exploring}, U-SAID~\cite{wang2019segmentation}, D2SM~\cite{mei2022deep}, WinNet~\cite{huang2022winnet}, and ADAP~\cite{jiang2022deep}. To ensure fair comparisons, we report our results on ${VR_{dia}}$ and ${VR_{seg}}$ with ResNet-50 and ResNet-101~\cite{he2016deep} as the backbone. Table~\ref{tab3} shows the results, indicating that our approach achieves state-of-the-art performance under different noise rates. In particular, \emph{GradProm} outperforms most existing methods by a large margin, even with a weaker backbone (ResNet-50 \emph{vs} ResNet-101), validating its importance and superiority. We achieve $36.173/0.971$, $28.024/0.901$, and $23.703/0.761$ PSNR/SSIM under noise rates of $0.1$, $0.2$, and $0.3$, respectively. Besides, we also observed an interesting phenomenon, where using ${VR_{seg}}$ as the auxiliary task yields better results than using ${VR_{dia}}$ under the setting that the IQE model is supervised. Through comparisons with experimental results on other datasets and the characteristics of the images in the dataset, we speculate that the reason for this may be that when the region of interest in the image is relatively large, ${VR_{seg}}$ can provide unabridged location guidance for the mainstream IQE model, and these regions are often difficult to handle during for IQE. In contrast, ${VR_{dia}}$ can only focus the visual recognition model on locally discriminative regions.

\myparagraph{Results on COVID-CT.} We compare denoising results on COVID-CT with state-of-the-art methods, including SR-ResNet~\cite{ledig2017photo}, ESTR~\cite{liu2022exploring} with ResNet-50 and ResNet-101~\cite{he2016deep}, NAFSSR~\cite{chu2022nafssr}, RLFNet~\cite{kong2022residual}, and DVANet~\cite{behjati2023single}. Our experiments use ${VR_{dia}}$ with ResNet-50 and ResNet-101 as the backbone. Table~\ref{tab4} shows the experimental results, indicating that our \emph{GradProm} achieves very competitive performance under different backbones, which further demonstrates its superiority. Although \emph{GradProm} achieves worse performance than DVANet~\cite{behjati2023single} in some cases of noise, DVANet is based on a stronger backbone than ours.

\myparagraph{Results on Lizard.}
We compare super-resolution performance on Lizard with state-of-the-art methods, including SR-ResNet~\cite{ledig2017photo}, ESTR~\cite{liu2022exploring}, U-SAID~\cite{wang2019segmentation}, D2SM~\cite{mei2022deep}, WinNet~\cite{huang2022winnet}, and ADAP~\cite{jiang2022deep}. Results in Table~\ref{tab5} indicates that our \emph{GradProm} achieves state-of-the-art performance with the same backbone. In comparison to the experimental results in Table~\ref{tab3} and Table~\ref{tab4}, it has been observed that the performance of the ${VR_{dia}}$, when employed as an auxiliary task, is superior to that of the ${VR_{seg}}$. This observation corroborates our previous conjecture. These results on different datasets and tasks validate the robustness and generalization of our \emph{GradProm}. Importantly, \emph{GradProm} does not require additional data or changes to the network architecture, so there is no increase in computational overhead in inference.

\myparagraph{Results on CAMUS.} 
We compare denoising results on CAMUS~\cite{leclerc2019deep} with state-of-the-art methods, including SR-ResNet~\cite{ledig2017photo}, ESTR~\cite{liu2022exploring}, DDM~\cite{stevens2023dehazing}, NRSF~\cite{wu2022semi}, and SUD~\cite{young2023supervision}. As shown in the Table~\ref{tab6} of experimental results, our \emph{GradProm} achieves new state-of-the-art performance under different noise ratios and with the assistance of different visual recognition models. Among the compared methods, we find that NRSF~\cite{wu2022semi} achieves better PSNR results than the other methods. However, as the noise ratio gradually increases, its SSIM metric exhibits a more significant decline. This may be due to the model focusing too much on improving certain regions that are easy to enhance, resulting in an imbalanced improvement in image quality across different regions. To address this problem, improvements to this method should focus on the region-of-interests through the design of the loss function. 
In the results of our proposed \emph{GradProm}, we obtained consistent conclusions as in ISIC 2018~\cite{codella2019skin}, indicating that when the region of interest (\ie, the endocardium regions are utilized in our experiments) in the image is relatively large, using a semantic segmentation model as the auxiliary model yields satisfactory results. 

\subsection{Visualizations}
In this section, we present visualized comparisons of our proposed method with state-of-the-art methods on ISIC 2018~\cite{codella2019skin}, COVID-CT~\cite{yang2020covid}, Lizard~\cite{graham2021lizard}, and CAMUS~\cite{leclerc2019deep}, as shown in Figure~\ref{fig4}. The results demonstrate that our model achieves the best visualization performance among the compared methods. Specifically, on the ISIC 2018 dataset, our method demonstrates significant performance in preserving the integrity of foreground objects and the sharpness of object boundaries compared to the state-of-the-art ADAP~\cite{jiang2022deep} and WinNet~\cite{huang2022winnet}. For instance, as shown in Figure~\ref{fig4} (a), our method almost achieves the same effect as ground-truth for some relatively small red lesion areas, validating the effectiveness of our proposed \emph{GradProm} approach. This finding suggests that our method can further improve the performance of mainstream IQE models while maintaining their original recognition performance. 
On the COVID-CT dataset, the visualizations in Figure~\ref{fig4} (b) show that our method achieves a more significant quality improvement compared to the second-best RLFNet~\cite{kong2022residual}, which even fails to recover color information of the image effectively. We hypothesize that this may be attributed to the specific characteristics of medical images, such as the fact that COVID-CT is a grayscale image dataset, and the method in question does not possess the capability to handle grayscale images effectively. These results also illustrate that our \emph{GradProm} can preserve the integrity of foreground objects, such as the lung, even in medical images. More importantly, our approach is not only capable of handling color medical images but also grayscale medical images, thereby validating the greater versatility of our method.
The visualized results on super-resolution in Figure~\ref{fig4} (c) show that our \emph{GradProm} outperforms existing methods, such as ADAP~\cite{jiang2022deep} and WinNet~\cite{huang2022winnet}, on the Lizard dataset. However, the results also show that our method cannot fully achieve results that are close to ground-truth, possibly due to the limited training number of images in Lizard. These findings suggest that both the model training strategy and the amount of training data are important factors affecting the model performance on this dataset. 
Finally, the visualization results on the CAMUS dataset are presented in Figure~\ref{fig4} (d). In comparison to the state-of-the-art methods on this dataset, namely SUD~\cite{young2023supervision} and NRSF~\cite{wu2022semi}, it can be observed that our \emph{GradProm} still achieves the best visualization results, characterized by clearer object boundaries and more distinct foreground object regions. 
Moreover, the visualization results also indicate that, compared to SUD and NRSF, our approach can mitigate the inherent speckle noise in ultrasound images. In comparison, SUD~\cite{young2023supervision} and NRSF~\cite{wu2022semi} tend to exacerbate the effects of speckle noise during the denoising process. This suggests that during training, the downstream visual recognition model can convey useful semantic information to the mainstream IQE model through our proposed strategy. This beneficial semantic information transcends the texture features in the images, resulting in superior perceptual outcomes.
These visualized comparisons presented in this section demonstrate the effectiveness and superiority of our proposed \emph{GradProm} over state-of-the-art methods on different image datasets and modalities.
\begin{figure}[t]
\centering
\includegraphics[width=.48\textwidth]{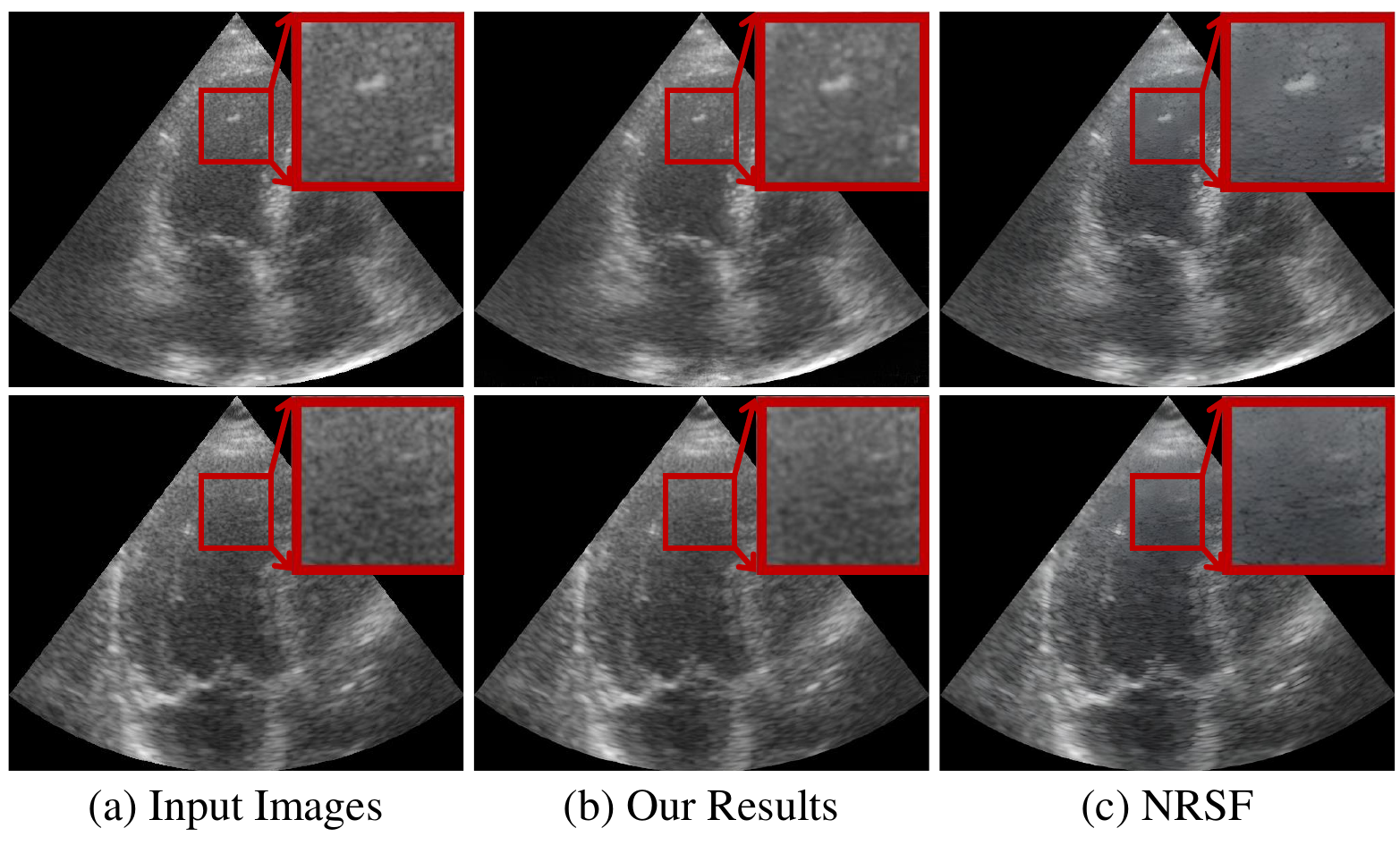}
\vspace{-2mm}
\caption{Visualization results on the genuine low-quality images. Samples are taken from the CAMUS~\cite{leclerc2019deep} dataset, which inherently exhibits the problem of speckle noise. A randomly highlighted red bounding-box emphasizes an area where our model demonstrates better results.}
\vspace{-2mm}
\label{fig6}
\end{figure} 

The inputs for the visualizations in Figure~\ref{fig4} are based on synthetic noisy samples. Additionally, we provide visualization results on genuine low-quality images in Figure~\ref{fig6}. We employ normal ultrasound images as input, which inherently exhibit the problem of speckle noise. For comparison purposes, we use the state-of-the-art method NRSF~\cite{wu2022semi} on this dataset. We can observe that our proposed \emph{GradProm} still demonstrates effectiveness in handling speckle noise, even though this type of noise was not considered during our training process. In contrast, NRSF exacerbates the effects of this noise. These visualization results on genuine low-quality ultrasound images confirm that the semantic information obtained through the downstream $VR$ model can indeed contribute to enhancing the performance of the mainstream IQE model.
\section{Conclusion and Future Work}
In this paper, we proposed a versatile and effective training strategy, \emph{GradProm}, for medical IQE within a task-driven framework. Applicable to a broad range of image modalities, \emph{GradProm} is not restricted to a specific medical image modality. Our method divides the IQE system into a mainstream model for image enhancement and an auxiliary model for visual recognition. The \emph{GradProm} strategy updates parameters only when the gradients of these two sub-models align in the same direction, as determined by their cosine similarity. If the gradients are not aligned, \emph{GradProm} employs the gradient of the mainstream image enhancement model for parameter updates. While more complex intermediate scenarios may exist in practice, we argue that investigating these particular details is unnecessary. The proposed \emph{GradProm} prevents the auxiliary visual recognition model from introducing bias into the mainstream image enhancement model during training. We have theoretically demonstrated that the gradient descent direction of the image enhancement model remains unbiased by the auxiliary visual recognition model when \emph{GradProm} is implemented. Furthermore, we have showcased the superiority of \emph{GradProm} over existing state-of-the-art methods in denoising, super-resolution, diagnosis, and segmentation tasks.

As a generalized method, \emph{GradProm} can be extended to other task-driven training processes, such as multi-objective learning, task-driven data augmentation, and image compression. It can be applied to the segmentation and classification branches of a semantic segmentation model, as well as the regression and detection branches of an object detection model. In future work, we aim to explore the effectiveness of \emph{GradProm} in other task-driven training processes. Besides, we plan to investigate the application of \emph{GradProm} to other medical image analysis tasks, including medical image registration and reconstruction. We also intend to examine the potential of \emph{GradProm} in combination with other technologies, such as transfer learning and domain adaptation, to further improve the generalization of trained models across diverse medical imaging modalities and datasets.
\bibliographystyle{IEEEtran}
\bibliography{main}
\end{document}